\documentclass[11pt]{article}

\usepackage{amsthm,amssymb, epsfig, epstopdf, tikz, tikz-cd, tcolorbox, subcaption, setspace, graphicx, enumitem, mathrsfs, bbm, mathtools,color,stmaryrd, comment, fancyheadings, mathpazo, euler, subcaption}

\theoremstyle{definition}

\usepackage[margin=.9in]{geometry}
\newtheorem{notation}{Notation}[section]

\newtheorem{definition}{Definition}[section]
\newtheorem{prop}{Proposition}[section]

\usepackage[section]{placeins}

\DeclareFixedFont{\ttb}{T1}{txtt}{bx}{n}{9} 
\DeclareFixedFont{\ttm}{T1}{txtt}{m}{n}{9}  
\usepackage{color} 
\definecolor{darkblue}{rgb}{0,0,.6}
\definecolor{darkred}{rgb}{.7,0,0}
\definecolor{darkgreen}{rgb}{0,.6,0}
\definecolor{red}{rgb}{.98,0,0}
\definecolor{deepblue}{rgb}{0,0,0.5}
\definecolor{deepred}{rgb}{0.6,0,0}
\definecolor{deepgreen}{rgb}{0,0.5,0}
\usepackage[colorlinks,pagebackref,pdfusetitle,urlcolor=darkblue,citecolor=darkblue,linkcolor=darkred,bookmarksnumbered,plainpages=false]{hyperref}
\usepackage{cleveref}

\usepackage{listings}
\newcommand\pythonstyle{\lstset{
language=Python,
basicstyle=\ttm,
morekeywords={self},              
keywordstyle=\ttb\color{deepblue},
emph={MyClass,__init__},          
emphstyle=\ttb\color{deepred},    
stringstyle=\color{deepgreen},
frame=tb,                         
showstringspaces=false
}}

\lstnewenvironment{python}[1][]
{
\pythonstyle
\lstset{#1}
}
{}

\newcommand\pythoninline[1]{{\pythonstyle\lstinline!#1!}}

\begin{document}
\title{Testing for Overfitting}
\author{James Schmidt\footnote{email: aschmi40@jhu.edu}\\
\small Johns Hopkins University Applied Physics Laboratory\\ 
\small Johns Hopkins University}
\date{\today}
\maketitle
\abstract{High complexity models are notorious in machine learning for overfitting, a phenomenon in which models well represent data but fail to generalize  an underlying data generating process. A typical  procedure for  circumventing overfitting computes empirical risk on a holdout set and halts once (or flags that/when) it begins to increase. Such practice often helps in outputting a well-generalizing model, but justification for why it works is usually heuristic. 


 We discuss the overfitting problem and explain why standard asymptotic and concentration results do not hold for evaluation with training data. We then proceed to introduce and argue for a  hypothesis test by means of which both model performance may be evaluated using training data, and overfitting quantitatively defined and detected. We rely on  said concentration bounds which guarantee that empirical means  should, with high probability, approximate their true mean to conclude that they should approximate each other. We stipulate conditions under which this test is valid,  explain why training data alone is unsuitable for generalization determination, describe how the test may be used for identifying overfitting, and articulate a further nuance with respect to which this test may flag distribution shift.}
\newcommand{\nat}[6][large]{%
  \begin{tikzcd}[ampersand replacement = \&, column sep=#1]
    #2\ar[bend left=40,""{name=U}]{r}{#4}\ar[bend right=40,',""{name=D}]{r}{#5}\& #3
          \ar[shorten <=10pt,shorten >=10pt,Rightarrow,from=U,to=D]{d}{~#6}
    \end{tikzcd}
}
\newcommand{\invamalg}{\mathbin{\text{\rotatebox[origin=c]{180}{$\amalg$}}}}

\newcommand{\dCrl}[0]{\mathfrak{dCrl}}
\newcommand{\ytil}{\tilde{y}}
\newcommand{\defeq}{\vcentcolon=}
\newcommand{\dee}{\partial}
\newcommand{\lb}{\{}
\newcommand{\rb}{\}}
\newcommand{\R}{\mathbb{R}}
\newcommand{\C}{\mathbb{C}}
\newcommand{\Q}{\mathbb{Q}}
\newcommand{\N}{\mathbb{N}}
\newcommand{\el}{\mathcal{L}}
\newcommand{\pdiv}[2]{\frac{\partial{#1}}{\partial{#2}}}
\newcommand{\discatp}{\displaystyle\bigsqcap}
\newcommand{\discats}{\displaystyle\bigsqcup}

\newcommand{\uZ}{\underline{\mathbb{Z}}}
\newcommand{\uF}[1]{\underline{\mathbb{F}}}
\newcommand{\one}{\mathbb{1}}
\newcommand{\two}{\mathbb{2}}

\newcommand{\dis}{\displaystyle}
\newcommand{\disp}{\displaystyle\prod}
\newcommand{\disu}{\displaystyle\bigcup}
\newcommand{\disi}{\displaystyle\bigcap}
\newcommand{\diss}{\displaystyle\sum}
\newcommand{\disg}{\displaystyle\int}
\newcommand{\disl}{\displaystyle\lim}
\newcommand{\dislim}{\displaystyle\lim}
\newcommand{\disliminf}{\displaystyle\liminf}
\newcommand{\dislimsup}{\displaystyle\limsup}
\newcommand{\disbop}{\displaystyle\bigotimes}
\newcommand{\disbos}{\displaystyle\bigoplus}
\newcommand{\dissup}{\displaystyle\sup}
\newcommand{\disinf}{\displaystyle\inf}
\newcommand{\dismax}{\displaystyle\max}
\newcommand{\dismin}{\displaystyle\min}
\newcommand{\dirlim}[1]{\displaystyle\varinjlim_{#1}}
\newcommand{\indlim}[1]{\displaystyle\varprojlim_{#1}}
\newcommand{\discatlim}{\indlim}
\newcommand{\discatcolim}{\dirlim}
\newcommand{\catcolim}{\mbox{colim}}
\newcommand{\catlim}{\mbox{lim}}

\newcommand{\colgray}[1]{\color{gray}{#1}\color{black}}

\newcommand{\sfM}{\sF{M}}

\newcommand{\forget}[2]{\Ub^{#1}_{#2}}

\newcommand\righttwoarrow{%
        \mathrel{\vcenter{\mathsurround0pt
                \ialign{##\crcr
                        \noalign{\nointerlineskip}$\rightarrow$\crcr
                        \noalign{\nointerlineskip}$\rightarrow$\crcr
                }%
        }}%
}

\newcommand{\Z}{\mathbb{Z}}
\newcommand{\Ab}[0]{\mathbb{A}}
\newcommand{\Bb}[0]{\mathbb{B}}
\newcommand{\Cb}[0]{\mathbb{C}}
\newcommand{\Db}[0]{\mathbb{D}}
\newcommand{\Eb}[0]{\mathbb{E}}
\newcommand{\Fb}[0]{\mathbb{F}}
\newcommand{\Gb}[0]{\mathbb{G}}
\newcommand{\Hb}[0]{\mathbb{H}}
\newcommand{\Ib}[0]{\mathbb{I}}
\newcommand{\Jb}[0]{\mathbb{J}}
\newcommand{\Kb}[0]{\mathbb{K}}
\newcommand{\Lb}[0]{\mathbb{L}}
\newcommand{\Mb}[0]{\mathbb{M}}
\newcommand{\Nb}[0]{\mathbb{N}}
\newcommand{\Ob}[0]{\mathbb{O}}
\newcommand{\Pb}[0]{\mathbb{P}}
\newcommand{\Qb}[0]{\mathbb{Q}}
\newcommand{\Rb}[0]{\mathbb{R}}
\newcommand{\Sb}[0]{\mathbb{S}}
\newcommand{\Tb}[0]{\mathbb{T}}
\newcommand{\Ub}[0]{\mathbb{U}}
\newcommand{\Vb}[0]{\mathbb{V}}
\newcommand{\Wb}[0]{\mathbb{W}}
\newcommand{\Xb}[0]{\mathbb{X}}
\newcommand{\Yb}[0]{\mathcal{Y}}
\newcommand{\Zb}[0]{\mathbb{Z}}

\newcommand{\sF}[1]{\mathsf{#1}}

\newcommand{\sC}[1]{\mathscr{#1}}

\newcommand{\mC}[1]{\mathcal{#1}}

\newcommand{\mB}[1]{\mathbb{#1}}

\newcommand{\mF}[1]{\mathfrak{#1}}

\newcommand{\Bc}[0]{\mathcal{B}}
\newcommand{\Cc}[0]{\mathcal{C}}
\newcommand{\Dc}[0]{\mathcal{D}}
\newcommand{\Ec}[0]{\mathcal{E}}
\newcommand{\Fc}[0]{\mathcal{F}}
\newcommand{\Gc}[0]{\mathcal{G}}
\newcommand{\Hc}[0]{\mathcal{H}}
\newcommand{\Ic}[0]{\mathcal{I}}
\newcommand{\Jc}[0]{\mathcal{J}}
\newcommand{\Kc}[0]{\mathcal{K}}
\newcommand{\Lc}[0]{\mathcal{L}}
\newcommand{\Mc}[0]{\mathcal{M}}
\newcommand{\Nc}[0]{\mathcal{N}}
\newcommand{\Oc}[0]{\mathcal{O}}
\newcommand{\Pc}[0]{\mathcal{P}}
\newcommand{\Qc}[0]{\mathcal{Q}}
\newcommand{\Rc}[0]{\mathcal{R}}
\newcommand{\Sc}[0]{\mathcal{S}}
\newcommand{\Tc}[0]{\mathcal{T}}
\newcommand{\Uc}[0]{\mathcal{U}}
\newcommand{\Vc}[0]{\mathcal{V}}
\newcommand{\Wc}[0]{\mathcal{W}}
\newcommand{\Xc}[0]{\mathcal{X}}
\newcommand{\Yc}[0]{\mathcal{Y}}
\newcommand{\Zc}[0]{\mathcal{Z}}

\newcommand{\aca}[0]{\mathcal{a}}
\newcommand{\bca}[0]{\mathcal{b}}
\newcommand{\cca}[0]{\mathcal{c}}
\newcommand{\dca}[0]{\mathcal{d}}
\newcommand{\eca}[0]{\mathcal{e}}
\newcommand{\fca}[0]{\mathcal{f}}
\newcommand{\gca}[0]{\mathcal{g}}
\newcommand{\hca}[0]{\mathcal{h}}
\newcommand{\ica}[0]{\mathcal{i}}
\newcommand{\jca}[0]{\mathcal{j}}
\newcommand{\kca}[0]{\mathcal{k}}
\newcommand{\lca}[0]{\mathcal{l}}
\newcommand{\mca}[0]{\mathcal{m}}
\newcommand{\nca}[0]{\mathcal{n}}
\newcommand{\oca}[0]{\mathcal{o}}
\newcommand{\pca}[0]{\mathcal{p}}
\newcommand{\qca}[0]{\mathcal{q}}
\newcommand{\rca}[0]{\mathcal{r}}
\newcommand{\sca}[0]{\mathcal{s}}
\newcommand{\tca}[0]{\mathcal{t}}
\newcommand{\uca}[0]{\mathcal{u}}
\newcommand{\vca}[0]{\mathcal{v}}
\newcommand{\wca}[0]{\mathcal{w}}
\newcommand{\xca}[0]{\mathcal{x}}
\newcommand{\yca}[0]{\mathcal{y}}
\newcommand{\zca}[0]{\mathcal{z}}

\newcommand{\Af}[0]{\mathfrak{A}}
\newcommand{\Bf}[0]{\mathfrak{B}}
\newcommand{\Cf}[0]{\mathfrak{C}}
\newcommand{\Df}[0]{\mathfrak{D}}
\newcommand{\Ef}[0]{\mathfrak{E}}
\newcommand{\Ff}[0]{\mathfrak{F}}
\newcommand{\Gf}[0]{\mathfrak{G}}
\newcommand{\Hf}[0]{\mathfrak{H}}
\newcommand{\If}[0]{\mathfrak{I}}
\newcommand{\Jf}[0]{\mathfrak{J}}
\newcommand{\Kf}[0]{\mathfrak{K}}
\newcommand{\Lf}[0]{\mathfrak{L}}
\newcommand{\Mf}[0]{\mathfrak{M}}
\newcommand{\Nf}[0]{\mathfrak{N}}
\newcommand{\Of}[0]{\mathfrak{O}}
\newcommand{\Pf}[0]{\mathfrak{P}}
\newcommand{\Qf}[0]{\mathfrak{Q}}
\newcommand{\Rf}[0]{\mathfrak{R}}
\newcommand{\Sf}[0]{\mathfrak{S}}
\newcommand{\Tf}[0]{\mathfrak{T}}
\newcommand{\Uf}[0]{\mathfrak{U}}
\newcommand{\Vf}[0]{\mathfrak{V}}
\newcommand{\Wf}[0]{\mathfrak{W}}
\newcommand{\Xf}[0]{\mathfrak{X}}
\newcommand{\Yf}[0]{\mathfrak{Y}}
\newcommand{\Zf}[0]{\mathfrak{Z}}

\newcommand{\af}[0]{\mathfrak{a}}
\newcommand{\bff}[0]{\mathfrak{b}}
\newcommand{\cf}[0]{\mathfrak{c}}
\newcommand{\dff}[0]{\mathfrak{d}}
\newcommand{\ef}[0]{\mathfrak{e}}
\newcommand{\ff}[0]{\mathfrak{f}}
\newcommand{\gf}[0]{\mathfrak{g}}
\newcommand{\hf}[0]{\mathfrak{h}}
\newcommand{\ifrak}{\mathfrak{i}}
\newcommand{\jf}[0]{\mathfrak{j}}
\newcommand{\kf}[0]{\mathfrak{k}}
\newcommand{\lf}[0]{\mathfrak{l}}
\newcommand{\mf}[0]{\mathfrak{m}}
\newcommand{\nf}[0]{\mathfrak{n}}
\newcommand{\of}[0]{\mathfrak{o}}
\newcommand{\pf}[0]{\mathfrak{p}}
\newcommand{\qf}[0]{\mathfrak{q}}
\newcommand{\rf}[0]{\mathfrak{r}}
\renewcommand{\sf}[0]{\mathfrak{s}}
\newcommand{\tf}[0]{\mathfrak{t}}
\newcommand{\uf}[0]{\mathfrak{u}}
\newcommand{\vf}[0]{\mathfrak{v}}
\newcommand{\wf}[0]{\mathfrak{w}}
\newcommand{\xf}[0]{\mathfrak{x}}
\newcommand{\yf}[0]{\mathfrak{y}}
\newcommand{\zf}[0]{\mathfrak{z}}
\newcommand{\cdX}{\mathfrak{cdX}}
\newcommand{\cdCrl}{\mathfrak{cdCrl}}

\newcommand{\scA}[0]{\mathscr{A}}
\newcommand{\scB}[0]{\mathscr{B}}
\newcommand{\scC}[0]{\mathscr{C}}
\newcommand{\scD}[0]{\mathscr{D}}
\newcommand{\scE}[0]{\mathscr{E}}
\newcommand{\scF}[0]{\mathscr{F}}
\newcommand{\scG}[0]{\mathscr{G}}
\newcommand{\scH}[0]{\mathscr{H}}
\newcommand{\scI}[0]{\mathscr{I}}
\newcommand{\scJ}[0]{\mathscr{J}}
\newcommand{\scK}[0]{\mathscr{K}}
\newcommand{\scL}[0]{\mathscr{L}}
\newcommand{\scM}[0]{\mathscr{M}}
\newcommand{\scN}[0]{\mathscr{N}}
\newcommand{\scO}[0]{\mathscr{O}}
\newcommand{\scP}[0]{\mathscr{P}}
\newcommand{\scQ}[0]{\mathscr{Q}}
\newcommand{\scR}[0]{\mathscr{R}}
\newcommand{\scS}[0]{\mathscr{S}}
\newcommand{\scT}[0]{\mathscr{T}}
\newcommand{\scU}[0]{\mathscr{U}}
\newcommand{\scV}[0]{\mathscr{V}}
\newcommand{\scW}[0]{\mathscr{W}}
\newcommand{\scX}[0]{\mathscr{X}}
\newcommand{\scY}[0]{\mathscr{Y}}
\newcommand{\scZ}[0]{\mathscr{Z}}

\newcommand{\fA}[0]{\mathsf{A}}
\newcommand{\fB}[0]{\mathsf{B}}
\newcommand{\fC}[0]{\mathsf{C}}
\newcommand{\fD}[0]{\mathsf{D}}
\newcommand{\fE}[0]{\mathsf{E}}
\newcommand{\fG}[0]{\mathsf{G}}
\newcommand{\fH}[0]{\mathsf{H}}
\newcommand{\fI}[0]{\mathsf{I}}
\newcommand{\fJ}[0]{\mathsf{J}}
\newcommand{\fK}[0]{\mathsf{K}}
\newcommand{\fL}[0]{\mathsf{L}}
\newcommand{\fM}[0]{\mathsf{M}}
\newcommand{\fN}[0]{\mathsf{N}}
\newcommand{\fO}[0]{\mathsf{O}}
\newcommand{\fP}[0]{\mathsf{P}}
\newcommand{\fQ}[0]{\mathsf{Q}}
\newcommand{\fR}[0]{\mathsf{R}}
\newcommand{\fS}[0]{\mathsf{S}}
\newcommand{\fT}[0]{\mathsf{T}}
\newcommand{\fU}[0]{\mathsf{U}}
\newcommand{\fV}[0]{\mathsf{V}}
\newcommand{\fW}[0]{\mathsf{W}}
\newcommand{\fX}[0]{\mathsf{X}}
\newcommand{\fY}[0]{\mathsf{Y}}
\newcommand{\fZ}[0]{\mathsf{Z}}

\newcommand{\fa }[0]{\mathsf{a}}
\newcommand{\fb }[0]{\mathsf{b}}
\newcommand{\fc }[0]{\mathsf{c}}
\newcommand{\fd}[0]{\mathsf{d}}
\newcommand{\fe}[0]{\mathsf{e}}
\newcommand{\fg}[0]{\mathsf{g}}
\newcommand{\fh}[0]{\mathsf{h}}
\newcommand{\fj}[0]{\mathsf{j}}
\newcommand{\fk}[0]{\mathsf{k}}
\newcommand{\fl }[0]{\mathsf{l}}
\newcommand{\fm }[0]{\mathsf{m}}
\newcommand{\fn }[0]{\mathsf{n}}
\newcommand{\fo }[0]{\mathsf{o}}
\newcommand{\fp}[0]{\mathsf{p}}
\newcommand{\fq}[0]{\mathsf{q}}
\newcommand{\fr}[0]{\mathsf{r}}
\newcommand{\fs}[0]{\mathsf{s}}
\newcommand{\ft }[0]{\mathsf{t}}
\newcommand{\fu }[0]{\mathsf{u}}
\newcommand{\fv }[0]{\mathsf{v}}
\newcommand{\fw}[0]{\mathsf{w}}
\newcommand{\fx}[0]{\mathsf{x}}
\newcommand{\fy}[0]{\mathsf{y}}
\newcommand{\fz}[0]{\mathsf{z}}

\newcommand{\dA}[0]{\dot{A}}
\newcommand{\dB}[0]{\dot{B}}
\newcommand{\dC}[0]{\dot{C}}
\newcommand{\dD}[0]{\dot{D}}
\newcommand{\dE}[0]{\dot{E}}
\newcommand{\dF}[0]{\dot{F}}
\newcommand{\dG}[0]{\dot{G}}
\newcommand{\dH}[0]{\dot{H}}
\newcommand{\dI}[0]{\dot{I}}
\newcommand{\dJ}[0]{\dot{J}}
\newcommand{\dK}[0]{\dot{K}}
\newcommand{\dL}[0]{\dot{L}}
\newcommand{\dM}[0]{\dot{M}}
\newcommand{\dN}[0]{\dot{N}}
\newcommand{\dO}[0]{\dot{O}}
\newcommand{\dP}[0]{\dot{P}}
\newcommand{\dQ}[0]{\dot{Q}}
\newcommand{\dR}[0]{\dot{R}}
\newcommand{\dS}[0]{\dot{S}}
\newcommand{\dT}[0]{\dot{T}}
\newcommand{\dU}[0]{\dot{U}}
\newcommand{\dV}[0]{\dot{V}}
\newcommand{\dW}[0]{\dot{W}}
\newcommand{\dX}[0]{\dot{X}}
\newcommand{\dY}[0]{\dot{Y}}
\newcommand{\dZ}[0]{\dot{Z}}

\newcommand{\da}[0]{\dot{a}}
\newcommand{\db}[0]{\dot{b}}
\newcommand{\dc}[0]{\dot{c}}
\newcommand{\dd}[0]{\dot{d}}
\newcommand{\de}[0]{\dot{e}}
\newcommand{\df}[0]{\dot{f}}
\newcommand{\dg}[0]{\dot{g}}
\renewcommand{\dh}[0]{\dot{h}}
\newcommand{\di}[0]{\dot{i}}
\renewcommand{\dj}[0]{\dot{j}}
\newcommand{\dk}[0]{\dot{k}}
\newcommand{\dl}[0]{\dot{l}}
\newcommand{\dm}[0]{\dot{m}}
\newcommand{\dn}[0]{\dot{n}}
\newcommand{\dq}[0]{\dot{q}}
\newcommand{\dr}[0]{\dot{r}}
\newcommand{\ds}[0]{\dot{s}}
\newcommand{\dt}[0]{\dot{t}}
\newcommand{\du}[0]{\dot{u}}
\newcommand{\dv}[0]{\dot{v}}
\newcommand{\dw}[0]{\dot{w}}
\newcommand{\dx}[0]{\dot{x}}
\newcommand{\dy}[0]{\dot{y}}
\newcommand{\dz}[0]{\dot{z}}

\newcommand{\oA}[0]{\overline{A}}
\newcommand{\oB}[0]{\overline{B}}
\newcommand{\oC}[0]{\overline{C}}
\newcommand{\oD}[0]{\overline{D}}
\newcommand{\oE}[0]{\overline{E}}
\newcommand{\oF}[0]{\overline{F}}
\newcommand{\oG}[0]{\overline{G}}
\newcommand{\oH}[0]{\overline{H}}
\newcommand{\oI}[0]{\overline{I}}
\newcommand{\oJ}[0]{\overline{J}}
\newcommand{\oK}[0]{\overline{K}}
\newcommand{\oL}[0]{\overline{L}}
\newcommand{\oM}[0]{\overline{M}}
\newcommand{\oN}[0]{\overline{N}}
\newcommand{\oO}[0]{\overline{O}}
\newcommand{\oP}[0]{\overline{P}}
\newcommand{\oQ}[0]{\overline{Q}}
\newcommand{\oR}[0]{\overline{R}}
\newcommand{\oS}[0]{\overline{S}}
\newcommand{\oT}[0]{\overline{T}}
\newcommand{\oU}[0]{\overline{U}}
\newcommand{\oV}[0]{\overline{V}}
\newcommand{\oW}[0]{\overline{W}}
\newcommand{\oX}[0]{\overline{X}}
\newcommand{\oY}[0]{\overline{Y}}
\newcommand{\oZ}[0]{\overline{Z}}

\newcommand{\oa}[0]{\overline{a}}
\newcommand{\ob}[0]{\overline{b}}
\newcommand{\oc}[0]{\overline{c}}
\newcommand{\od}[0]{\overline{d}}
\renewcommand{\oe}[0]{\overline{e}}
\newcommand{\og}[0]{\overline{g}}
\newcommand{\oh}[0]{\overline{h}}
\newcommand{\oi}[0]{\overline{i}}
\newcommand{\oj}[0]{\overline{j}}
\newcommand{\ok}[0]{\overline{k}}
\newcommand{\ol}[0]{\overline{l}}
\newcommand{\om}[0]{\overline{m}}
\newcommand{\on}[0]{\overline{n}}
\newcommand{\oo}[0]{\overline{o}}
\newcommand{\op}[0]{\overline{p}}
\newcommand{\oq}[0]{\overline{q}}
\newcommand{\os}[0]{\overline{s}}
\newcommand{\ot}[0]{\overline{t}}
\newcommand{\ou}[0]{\overline{u}}
\newcommand{\ov}[0]{\overline{v}}
\newcommand{\ow}[0]{\overline{w}}
\newcommand{\ox}[0]{\overline{x}}
\newcommand{\oy}[0]{\overline{y}}
\newcommand{\oz}[0]{\overline{z}}

\renewcommand{\a}{\alpha}
\renewcommand{\b}{\beta}
\renewcommand{\d}{\delta}
\newcommand{\e}{\varepsilon}
\newcommand{\f}{\phi}
\newcommand{\g}{\gamma}
\newcommand{\h}{\eta}
\renewcommand{\i}{\iota}
\renewcommand{\k}{\kappa}
\renewcommand{\l}{\lambda}
\newcommand{\m}{\mu}
\newcommand{\n}{\nu}
\newcommand{\p}{\pi}
\newcommand{\ph}{\varphi}
\newcommand{\ps}{\psi}
\newcommand{\q}{\xi}
\renewcommand{\r}{\rho}
\newcommand{\s}{\sigma}
\renewcommand{\t}{\tau}
\renewcommand{\v}{\upsilon}
\newcommand{\x}{\chi}
\newcommand{\z}{\zeta}
\newcommand{\G}{\Gamma}

\newcommand{\aarb}[0]{\<a>}
\newcommand{\barb}[0]{\<b>}
\newcommand{\carb}[0]{\<c>}
\newcommand{\darb}[0]{\<d>}
\newcommand{\earb}[0]{\<e>}
\newcommand{\farb}[0]{\<f>}
\newcommand{\garb}[0]{\<g>}
\newcommand{\harb}[0]{\<h>}
\newcommand{\iarb}[0]{\<i>}
\newcommand{\jarb}[0]{\<j>}
\newcommand{\karb}[0]{\<k>}
\newcommand{\larb}[0]{\<l>}
\newcommand{\marb}[0]{\<m>}
\newcommand{\narb}[0]{\<n>}
\newcommand{\oarb}[0]{\<o>}
\newcommand{\parb}[0]{\<p>}
\newcommand{\qarb}[0]{\<q>}
\newcommand{\rarb}[0]{\<r>}
\newcommand{\sarb}[0]{\<s>}
\newcommand{\tarb}[0]{\<t>}
\newcommand{\uarb}[0]{\<u>}
\newcommand{\varb}[0]{\<v>}
\newcommand{\warb}[0]{\<w>}
\newcommand{\xarb}[0]{\<x>}
\newcommand{\yarb}[0]{\<y>}
\newcommand{\zarb}[0]{\<z>}

\newcommand{\hA}[0]{\hat{A}}
\newcommand{\hB}[0]{\hat{B}}
\newcommand{\hC}[0]{\hat{C}}
\newcommand{\hD}[0]{\hat{D}}
\newcommand{\hE}[0]{\hat{E}}
\newcommand{\hF}[0]{\hat{F}}
\newcommand{\hG}[0]{\hat{G}}
\newcommand{\hH}[0]{\hat{H}}
\newcommand{\hI}[0]{\hat{I}}
\newcommand{\hJ}[0]{\hat{J}}
\newcommand{\hK}[0]{\hat{K}}
\newcommand{\hL}[0]{\hat{L}}
\newcommand{\hM}[0]{\hat{M}}
\newcommand{\hN}[0]{\hat{N}}
\newcommand{\hO}[0]{\hat{O}}
\newcommand{\hP}[0]{\hat{P}}
\newcommand{\hQ}[0]{\hat{Q}}
\newcommand{\hR}[0]{\hat{R}}
\newcommand{\hS}[0]{\hat{S}}
\newcommand{\hT}[0]{\hat{T}}
\newcommand{\hU}[0]{\hat{U}}
\newcommand{\hV}[0]{\hat{V}}
\newcommand{\hW}[0]{\hat{W}}
\newcommand{\hX}[0]{\hat{X}}
\newcommand{\hY}[0]{\hat{Y}}
\newcommand{\hZ}[0]{\hat{Z}}

\newcommand{\ha}[0]{\hat{a}}
\newcommand{\hb}[0]{\hat{b}}
\newcommand{\hc}[0]{\hat{c}}
\newcommand{\hd}[0]{\hat{d}}
\newcommand{\he}[0]{\hat{e}}
\newcommand{\hg}[0]{\hat{g}}
\newcommand{\hh}[0]{\hat{h}}
\newcommand{\hi}[0]{\hat{i}}
\newcommand{\hj}[0]{\hat{j}}
\newcommand{\hk}[0]{\hat{k}}
\newcommand{\hl}[0]{\hat{l}}
\newcommand{\hm}[0]{\hat{m}}
\newcommand{\hn}[0]{\hat{n}}
\newcommand{\ho}[0]{\hat{o}}
\newcommand{\hp}[0]{\hat{p}}
\newcommand{\hq}[0]{\hat{q}}
\newcommand{\hr}[0]{\hat{r}}
\newcommand{\hs}[0]{\hat{s}}
\newcommand{\hu}[0]{\hat{u}}
\newcommand{\hv}[0]{\hat{v}}
\newcommand{\hw}[0]{\hat{w}}
\newcommand{\hx}[0]{\hat{x}}
\newcommand{\hy}[0]{\hat{y}}
\newcommand{\hz}[0]{\hat{z}}

\newcommand{\hyph}[2]{\Fc_{#1}:#2\rightarrow \sF{RelMan^c}}
\newcommand{\repsys}{(\Fc_{N^\scT}:\N\rightarrow\sF{RelMan^c},\frac{d}{dt})}
\newcommand{\truerep}[1]{\left(\Fc_{N^\scT}{#1}:\N\rightarrow\sF{RelMan^c},\frac{D}{Dt}\right)}

\section{Introduction}

Supervised machine learning is  severely underdetermined: a finite labeled data set is used to search a function space for an appropriate model fitting both the data and ``from where the data comes.'' While the full function space is often at least two infinite orders of magnitude greater than the data, practitioners typically restrict search to a hypothesis class that is parametrized as a finite dimensional space. If this hypothesis class is too restricted, the search may output a model which fails to represent or approximate the data well enough; if, on the other hand, the class is too rich, the output model may represent the data \textit{too} well, in that the model fails to represent the underlying distribution from which data is drawn. Generally, this tradeoff between \textit{underfitting} and \textit{overfitting}, respectively, is asymmetric: a model which fits data  may (and hopefully does) still  generalize to the underlying distribution, while a model which underfits data usually does not fit the distribution. Stated differently, underfitting is \textit{detectable} in the course of performance evaluation while overfitting cannot be identified by performance on the training data alone (\cite{ghojogh}).

To mitigate the aspect blindness of training data performance to overfitting,  standard practice sets aside a holdout set disjoint from training and computes performance separately. Thus, training a model ordinarily incorporates two distinct steps: 1.\ optimization with training data to fit (model to) data and 2.\ verification of generalization by evaluating performance on holdout data.  While vague heuristics motivating this two-step procedure abound in the literature and research community, rigorous statistical rationale less ubiquitously accompany justification of its use.   Moreover, this two-step process facially treats training data and holdout data as altogether different kinds of things, with different tasks and different intended uses. As such, separating the conclusions we draw from training data and holdout data threatens to undermine the  original impetus according to which training data is used for training in the first place, namely \textit{that} optimization with respect to training data should \textit{thereby} optimize an expectation (generalization). We explain the reasons for this paradox, and propose a solution that translates into a statistical test which may be deployed for both defining and identifying overfitting, using modified Law of Large Numbers (LLN) intuition that empirical means should approximate their expectation.

 In \cref{section:background}, we review requisite background for the supervised learning problem, discuss the problem with training data, how it relates to overfitting, and why we would still like to use model performance on training data to asses generalization. In \cref{section:theTest}, we detail the statistical test for achieving this end,  and give commentary on how this test clarifies the  meaning of overfitting. We point out how the test validates generalization even absent strong but restrictive (e.g.\ PAC) learnability guarantees.   
\section{Technical Background}\label{section:background}
\subsection{Supervised Machine Learning} 
The setting for a supervised machine learning problem starts with the following data: 
\begin{enumerate}
	\item a joint probability space $(\Xc\times\Yc,\Pb_{\Xc\times\Yc})$,\footnote{We leave implicit the $\sigma$-algebra of measurable sets and suppose that anything we try measuring is indeed $\Pb$-measurable.}
	\item labeled data $\sF{S} = \big((x_1,y_1),\ldots,(x_m,y_m)\big)\in (\Xc\times\Yc)^\omega\defeq \disu_{m\in \N} (\Xc\times\Yc)^m$, 
	\item a hypothesis class $\Hc \subset \Yc^\Xc$ of functions $\tilde{y}:\Xc\rightarrow \Yc$,\footnote{The notation $\Yc^\Xc$ denotes the \textit{set} $\big\{\tilde{y}:\Xc\rightarrow \Yc\big\}$ of unstructured functions with domain $\Xc$ and codomain $\Yc$. Of course, we require $\Hc$ to consist only of measurable such functions.} usually finite dimensional, elements $\tilde{y}\in\Hc$ of which  are  called \textit{models},  and
	\item a cost function generator $c:\Hc\rightarrow \R^{\Xc\times\Yc}$ mapping a model $\tilde{y}$ to random variable $c_{\tilde{y}}:\Xc\times\Yc\rightarrow\R$, whose output $c_{\tilde{y}}(x,y)$ on input $(x,y)$ is a measure of fit between prediction $\tilde{y}(x)$ and label $y$. 
\end{enumerate}
\begin{notation}
In the subsequent, we consolidate notation with  $\Zc\defeq \Xc\times\Yc$ and $z=(x,y)\in \Zc$. 	
\end{notation}

The goal is to concoct an \textit{algorithm} $\hy_{(\cdot)}:\Zc^\omega \rightarrow \Hc$ which outputs a model $\hy_{\sF{S}}$ with small expected cost $$\Eb(c_{\hy_{\sF{S}}}) \approx \disinf_{\tilde{y}\in \Hc}\Eb(c_{\tilde{y}}),$$ having some guarantees of approximation performance in probability. The measure $\Pb_{\Xc\times\Yc}$ generating data $(x_i,y_i)$  is usually unknown, and data $\sF{S}$ is used to proxy approximate expectation and to optimize the expected risk function \begin{equation}
  \begin{array}{lrcl} \Eb(c_{(\cdot)}):& \Hc &\rightarrow  & \R \\
  &\tilde{y}&\mapsto & \Eb(c_{\tilde{y}}). 	
  \end{array}\end{equation} 
The standard algorithm for  this optimization is empirical risk minimization, namely \begin{equation}\label{eq:erm} \hy_{\sF{S}}\in \displaystyle\arg\min_{\tilde{y}\in \Hc}e_{\sF{S}}(c_{\tilde{y}}),\end{equation} where empirical risk is defined as \begin{equation}\label{eq:empiricalRisk} e_{\sF{S}}(c_{\tilde{y}})\defeq \frac{1}{|\sF{S}|}\diss_{(x,y)\in \sF{S}}c_{\tilde{y}}(x,y).\footnote{When cost $c$ is fixed and clear, we also denote $e_{\sF{S}}(c_{\tilde{y}})$ simply by $e_{\sF{S}}(\tilde{y})$.}\end{equation} Law of Large Numbers intuition suggests that \begin{equation}\label{eq:lln1} e_{\sF{S}}(c_{\tilde{y}})\approx \Eb(c_{\tilde{y}})\end{equation} when $|\sF{S}|$ is large, so supposing as much, an output $\hy_{\sF{S}}$ of \cref{eq:erm} ought to be a close approximation of the true goal, in the sense that \begin{equation}\label{eq:generalization}\Eb(c_{\hy_{\sF{S}}})\approx \dis\inf_{\tilde{y}\in \Hc}\Eb(c_{\tilde{y}}).\end{equation}

To the extent that a model $\tilde{y}\in \Hc$ (approximately) satisfies approximation \eqref{eq:lln1}, we say that the model \textit{generalizes} ($\e$-generalizes if the error in approximation is bounded by $\e$), and to the extent that models in $\Hc$ can be guaranteed to generalize optimality $\inf_{\tilde{y}}\Eb(c_{\tilde{y}})$ \eqref{eq:generalization}, we say that $\Hc$ is some kind of \textit{learnable}. The familiar and formal notion of \textit{probably approximately correct} (PAC) learnability, for example, extends guarantees of concentration bounds to an  optimization (over $\Hc$) context, and defines $\Hc$ to be PAC learnable if there is a sample complexity $\mu:(0,1)^2\rightarrow\N$ for which $\hy_{\sF{S}}$ may be guaranteed to $\e$-generalize  with at least $1-\d$ probability as long as $|\sF{S}|>\mu(\e,\d)$ (\cite{uml}, \cite{fml}).\footnote{Explicitly, if $m>\mu(\e,\d)$ then $\Pb_{(\Xc\times\Yc)^m}\left(\big|\Eb(c_{\hy_{(\cdot)}}) - \inf_{\tilde{y}\in \Hc}\Eb(c_{\tilde{y}})\big|>\e\right)<\d $. Strictly speaking, PAC learnability only requires the existence of an algorithm $\hy:(\Xc\times\Yc)^\omega\rightarrow\Hc$ satisfying this bound, not necessarily that empirical risk minimization is it.} Properly quantifying the character  and richness of $\Hc$ (as captured, e.g., by VC dimension) demarcates learnability conditions, and various theoretical results exist providing such guarantees.

\subsection{Overfitting and Generalization}\label{subsection:overfitting}
Absent formal learnability guarantees, it turns out that LLN reasoning is not sufficient for ensuring generalization. The reasons are multifarious but substantively turn around \textit{currying} (\cite[\S2.3]{riehl}) of the cost function generator $c:\Hc\rightarrow\R^{\Zc}$.  Currying  amounts to pre-fixing arguments of a multivariable function to generate a function of fewer variables, and casting our situation in this formalism is helpful for understanding the problem. For a \textit{fixed} model $\tilde{y}\in \Hc$, the map $c_{\tilde{y}}:\Zc\rightarrow \R$ is a random variable, and induces a measure $\Pb_\R$ on $\R$ by $\Pb_\R([a,b])\defeq \Pb_{\Xc\times\Yc}\big(c_{\tilde{y}}^{-1}([a,b])\big)$. This means, among other things,  given data $\sF{S}=\big(z_1,\ldots,z_m)\big)\sim_{iid}\Pb_{\Zc}$, that  $c_{\tilde{y}}(\sF{S})\defeq\big(c_{\tilde{y}}(z_1),\ldots,c_{\tilde{y}}(z_m)\big)\sim_{iid}\Pb_\R$. And  independence invites valid conclusions of various concentration results. 
  
\begin{figure}[h!]
\centerline{\includegraphics[scale=.2]{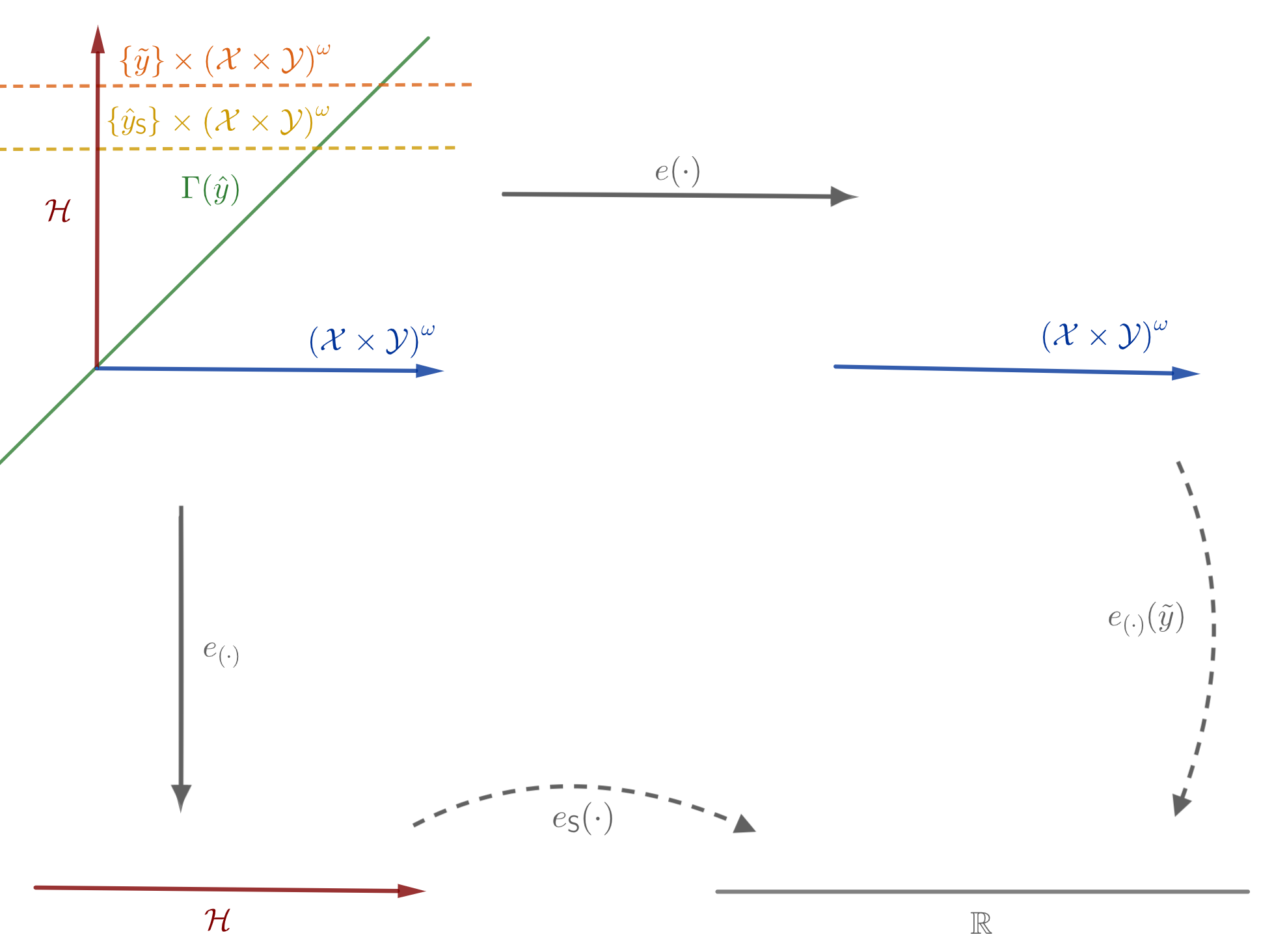}}\caption{Slicing the Cost Function $c:\Hc\times(\Xc\times\Yc)^\omega\rightarrow\R$}\label{fig:figure1}
\end{figure}

Searching over a function space, in the supervised machine learning setting, adds complications to otherwise innocuous independence conclusions. For the learning algorithm $\hy:\Zc^\omega\rightarrow\Hc$ first takes \textit{data} $\sF{S}\in \Zc^\omega$ in search of a certain minimum with respect to \textit{this} data. Given different data, the algorithm outputs a different model. The curried cost generator, by contrast, $c_{(\cdot)}:\Hc\rightarrow \R^{\Zc}$ defines an empirical risk generator $e(\cdot):\Hc\rightarrow \R^{(\Zc)^\omega}$ defined by sending $\tilde{y}\mapsto e_{(\cdot)}(c_{\tilde{y}})$, the latter of which is defined by mapping data $\sF{S}\in (\Zc)^\omega$ to $e_{\sF{S}}(c_{\tilde{y}})$  (\cref{eq:empiricalRisk}), and with respect to which LLN reasoning and the like may properly apply. The learning optimization procedure, however, flips the currying around: fixing data $(x,y)\in \Xc\times\Yc$, we have a cost on models $c_{(\cdot)}(x,y):\Hc\rightarrow\R$ defined by $\tilde{y}\mapsto c_{\tilde{y}}(x,y)$, which extends to empirical risk $e_{\sF{S}}(\cdot):\Hc\rightarrow\R$ mapping model $\tilde{y}\mapsto e_{\sF{S}}(c_{\tilde{y}})$, instantiating the curried function $e_{(\cdot)}:\Zc^\omega\rightarrow\R^\Hc$.\footnote{The reversal of roles in subscripts between $c$ and $e$ is unfortunate, but otherwise reflective of the primary purpose of each function, namely that $c_{\tilde{y}}$ measures performance of model $\tilde{y}$ on a datapoint $z\in \Zc$ while $e_{\sF{S}}$ measures empirical risk of fixed data on a model $\tilde{y}\in \Hc$.}

Order of operations matter.  Consider  uncurried versions $\begin{tikzcd}\Hc\times\Zc^\omega\arrow[r,shift left,"c_{(\cdot)}(\cdot)"]\arrow[r,shift right,swap,"e_{(\cdot)}(c_{(\cdot))}"] & \R \end{tikzcd}$  of the cost and empirical risk functions---thereby deprioritizing either data $\sF{S}\in\Zc^\omega$ or model $\tilde{y}\in \Hc$---and define  $$\G(\hy)\defeq \big\{(\tilde{y},\sF{S})\in \Hc\times\Zc^\omega:\, \tilde{y} = \hy_{\sF{S}}\big\},$$ the pullback of diagram $\begin{tikzcd}\Hc\times\Zc^\omega\arrow[r,shift left,"\hy_{(\cdot)}\circ \pi_2"]\arrow[r,shift right,swap,"id_\Hc\circ \pi_1"] & \Hc \end{tikzcd}$ (see \cref{fig:figure1}). Evaluation of  empirical performance for a model $\hy_{\sF{S}}$ using training data $\sF{S}$ lives in  $\G(\hy)$, which elucidates why performance of training data is ``aspect blind'' to overfitting: Law of Large Numbers reasoning does not apply in this regime.

To see why not, consider that a sequence of datasets $$\begin{array}{lcl}\sF{S_1}  & = & z_1\in\Zc^1,\\  \sF{S_2} & = &\big(z_1,z_2\big)\in \Zc^2,\\  & \vdots &  \\ 
\sF{S}_m& = &\big(z_1,\ldots,z_m\big)\in\Zc^m,\\  & \vdots & \end{array}$$ with each $\sF{S}_j\sim_{iid}\Pb_{\Zc}$, induces a sequence of models $$\hy_{\sF{S}_1}, \hy_{\sF{S}_2},\ldots,\hy_{\sF{S}_m},\ldots \in \Hc.$$ The sequence of models consequently induces a sequence of finite sequences of costs \begin{equation}\label{eq:seqCosts}\begin{array}{lcl} c_{\hy_{\sF{S}_1}}(\sF{S}_1)& = & c_{\hy_{\sF{S}_1}}(z_1)\in \R, \\ c_{\hy_{\sF{S}_2}}(\sF{S}_2)& = & \big(c_{\hy_{\sF{S}_2}}(z_1),c_{\hy_{\sF{S}_2}}(z_2)\big)\in\R^2,\\ & \vdots & \\ c_{\hy_{\sF{S}_m}}(\sF{S}_m) & = & \big(c_{\hy_{\sF{S}_m}}(z_1),\ldots,c_{\hy_{\sF{S}_m}}(z_m)\big)\in \R^m,\\ & \vdots & \end{array}\end{equation} which clearly is  not guaranteed to be iid, unless miraculously the cost functions \begin{equation}\label{eq:seqCostsNIID} c_{\hy_{\sF{S}_1}}, c_{\hy_{\sF{S}_2}}, \ldots, c_{\hy_{\sF{S}_m}},\ldots\end{equation} all induce the same measure, say $\Pb_{c_{\hy}(\Xc\times\Yc)}$, on $\R$, for which there is no apriori reason to suppose; for different models $\tilde{y}'\neq \tilde{y}\in \Hc$ are not guaranteed to induce the same cost functions $c_{\tilde{y}'}$, $c_{\tilde{y}}$. 

There is, however, a sense in which sequence of costs $c_{\sF{S}_1},\ldots,c_{\sF{S}_k}$, for models $\hy_{\sF{S}_1},\ldots,\hy_{\sF{S}_k}\in\Hc$, \textit{is} iid, though care must be taken to identify the correct measure for which this is so. The restrictions $\hy|_{\Zc^j}:\Zc^j\rightarrow \Hc$ induce a collection of measures $\Pb_\Hc^j$ on $\Hc$  for $j\in \N$. So for \textit{fixed} $m\in\N$ and iid sequence $\sF{S}_1,\ldots,\sF{S}_k\sim_{iid}\Pb_{\Zc^m}$, we then have iid sequence $\hy_{\sF{S_1}},\ldots,\hy_{\sF{S_k}}\sim_{iid}\Pb_\Hc$  and therefore too iid sequence \begin{equation}\label{eq:seqCostsIID} c_{\hy_{\sF{S}_1}},\ldots,c_{\hy_{\sF{S_k}}}\sim_{iid} \Pb_{\R}\end{equation} of random variables. In \eqref{eq:seqCostsNIID}, we consider each $c_{\hy_{\sF{S_j}}}:\Zc\rightarrow\R$ as a different random variable, and expect, among other things, that $\Eb(c_{\hy_{\sF{S_j}}})\neq \Eb(c_{\hy_{\sF{S_i}}})$ for $i\neq j$; iidness in \eqref{eq:seqCostsIID} by contrast implies equality $\Eb(c_{\hy_{\sF{S_j}}})= \Eb(c_{\hy_{\sF{S_i}}})$ even when $i\neq j$. But in the latter case  we measure and integrate with respect to $\Pb_{\Zc^m}$, and in the former case with respect to $\Pb_{\Zc}$, with model $\hy_{\sF{S}}\in\Hc$ fixed.\footnote{The expectation $\disg_{\Xc\times\Yc}c_{\hy_{\sF{S}}}(x,y)d\Pb_{\Zc}(z)$ is a number while the expectation $\disg_{\Zc^m}c_{\hy_{\sF{S}}}d\Pb_{\Zc^m}(\sF{S})$ is itself a random variable.} Said once more,  realizations of $\hy_{\sF{S}}$ for data set $\sF{S}$ may be dually interpreted as instantiating a new random variable $\Zc\xrightarrow{c_{\hy_{\sF{S}}}}\R$---for which different samplings realize different random variables---or as instancing particular outcomes of the same random function $\Zc^m\xrightarrow{c_{\hy_{(\cdot)}}}\Hc$.

 \begin{figure}[h!]
\centerline{\includegraphics[scale=.175]{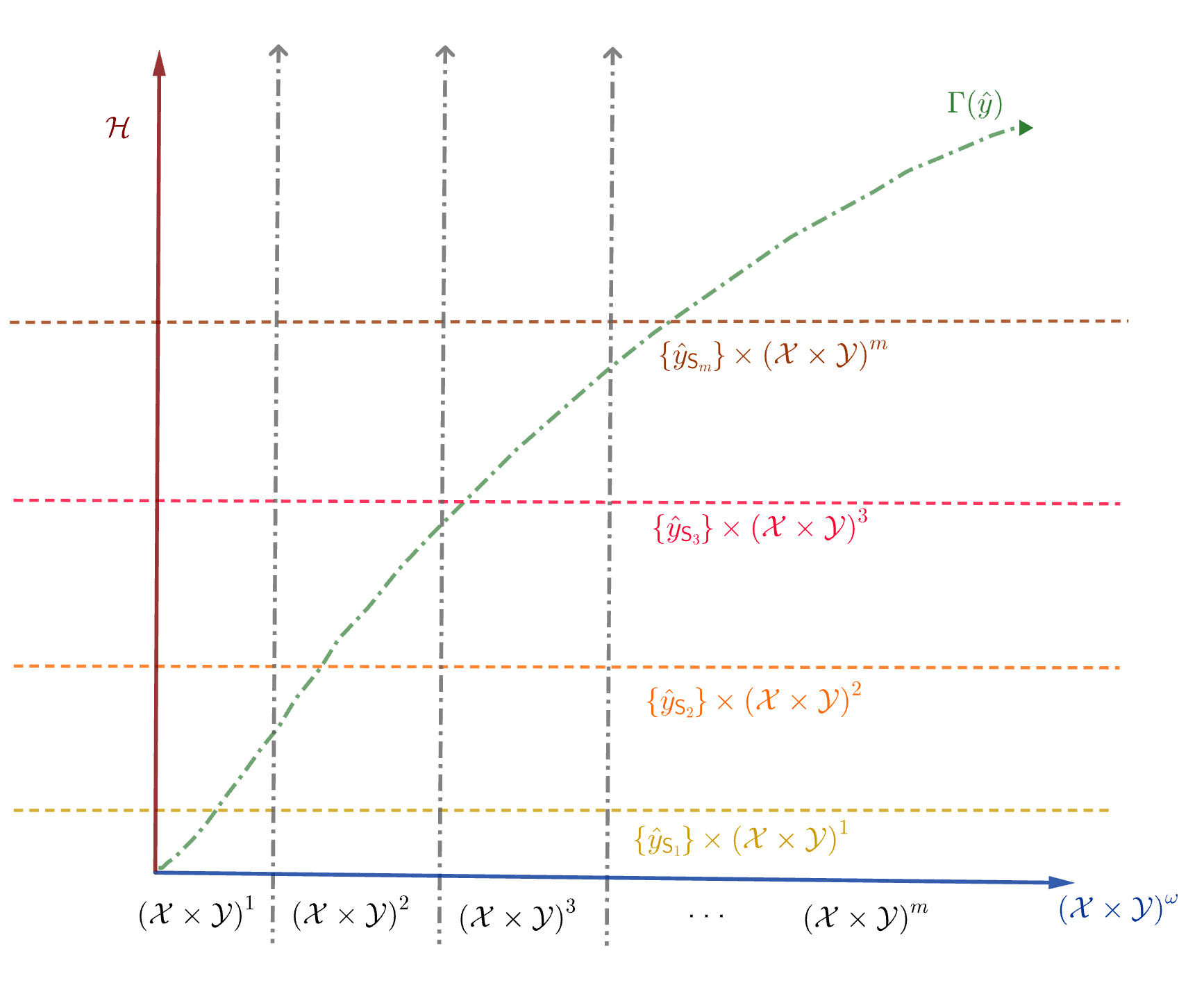}}\caption{Sequences of models $\hy_{\sF{S_1}},\hy_{\sF{S_2}},\ldots,\hy_{\sF{S_m}},\ldots$} \label{fig:figure2}
\end{figure}

To sort resolve the conflict, consider the map $(\Xc\times\Yc)^{m+m'}\xrightarrow{c_{\hy_{(\cdot)}(\cdot)}}\R^{m'}$ defined by sending \begin{equation}\label{eq:doubleCost} (\sF{S},\sF{S}')\mapsto c_{\hy_{\sF{S}}}(\sF{S}')\defeq \big(c_{\hy_{\sF{S}}}(z_1'),\ldots,c_{\hy_{\sF{S}}}(z_{m'}')\big)\in \R^k\end{equation} with $\sF{S}'=  \big(z_1',\ldots,z_{m'}'\big)$.  Independence of evaluated cost samples in \eqref{eq:seqCosts} may then be recovered by first fixing model $\hy_{\sF{S}}\in \Hc$ and separately evaluating  $$\big(c_{\hy_{\sF{S}}}(z_1'),\ldots,c_{\hy_{\sF{S}}}(x_{m'}',y_{m'}')\big)\in \R^{m'},$$ which corresponds to the evaluating the curryied map $c_{\hy_{\sF{S}}}(\cdot):\Zc^{m'}\rightarrow \R^{m'}$. 
 In this way, we isolate our attention to slices  $\{\hy_{\sF{S}}\}\times \Zc^\omega$ (see \cref{fig:figure2}), for which a sequence of  samples $\sF{S}'\sim_{iid}\Pb_{\Zc}$ induces truly independent and identically distributed sample $c_{\hy_{\sF{S}}}(\sF{S}')\sim_{iid}\Pb_{c_{\hy_{\sF{S}}}(\Zc)}$. 
Then  for each such $\hy_{\sF{S}_j}$ above, LLN holds in the sense that there is, for $\e,\d>0$, a number $m'>0$ for which \begin{equation}\label{eq:validation}\Pb_{\Zc^{m'}}\left(\big|e_{(\cdot)}(\hy_{\sF{S}_j}) - \Eb(c_{\hy_{\sF{S}_j}})\big|>\e\right)<\d\end{equation} whenever $m>m_j$.  Now the problem of performance evaluation  with $\sF{S}$  from \eqref{eq:seqCosts} crystalizes: with composed map \begin{equation}\label{eq:curriedComposedCost} \Zc^m\hookrightarrow\Zc^{2m}\xrightarrow{c_{\hy_{(\cdot)}}(\cdot)}\R^m\end{equation} taking $\sF{S}\mapsto c_{\hy_{\sF{S}}}(\sF{S})$, the induced measure on $\R^m$ arises from (the corresponding portion of) a map $\G(\hy)\rightarrow\R^m$, instead of the slice $\{\hy_{\sF{S}}\}\times\Zc^m$ in \cref{fig:figure2}. It is simply a different measure!

This discussion formalizes uses of data for training, validation, and testing or evaluation. The inequality in \eqref{eq:validation} is common (see e.g.\ \cite[Theorem 11.1]{uml}) and provides a test for quantifying a model's generalization performance. The second data set $\sF{S}'$ is commonly called `validation' or `test' data, according to its end in the training pipeline.  Typically performance at each training stage is evaluated on the holdout set, and early stopping conditions verify that validation performance continues to improve \cite{prechelt}. An onset of  validation performance degradation can be interpreted as indication of overfitting.  Illustrations of overfitting in the literature (e.g. \cite{bilmes}, \cite{ghojogh}, \cite{o2o}) display performance on training data compared with performance on holdout data, often parameterized by model complexity or training step (\cite{roelofs2}, \cite{roelofs}).  Caution must be taken when tuning or other decisioning is guided by performance on validatation data: model hyperparameters are often obtained by minimizing $e_{\sF{S}'}(\hy_{\sF{S}})$ with respect to them over $\sF{S}'$. Doing so, however, lands one in exactly the same situation as described previously of measuring along $\G(\hy)$ instead of $\{\tilde{y}\}\times\Zc^x$. From the measure's perspective, therefore, there is little difference between training with e.g.\ gradient descent and updating hyperparameters with cross-validation: to obtain an honest iid sequence of cost samples, one once and for all fixes the model with respect to which empirical performance is to be evaluated.

\section{Quantifying Overfitting}\label{section:theTest}
\subsection{Empirical Risk Estimation} 
We consider  only the case where cost $c_{(\cdot)}(\cdot):\Hc\times\Zc\rightarrow\R$ is bounded as  $c_{\tilde{y}} \subset [0,1]$, such as most classification problems or  restricted classes of regression problems. In this case, Hoeffding-like bounds abound and we expect that \begin{equation}\label{eq:expectedApprox}
 \Pb_{\Zc^{m'}}\left(\left|\Eb(c_{\hy_{\sF{S}}}) - e_{(\cdot)}(\hy_{\sF{S}})\right|>\e \right) < 2e^{-2\e^2m'}.
 \end{equation}
 
In other words, for  independently and identically distributed sampled data $\sF{S}'\in \Zc^{m'}$, $e_{\sF{S}'}(\hy_{\sF{S}}) \approx \Eb(c_{\hy_{\sF{S}}})$ ($\pm \e$) with probability  at least $1-e^{-2\e^2{m'}}$.\footnote{The fact that $\Eb(c_{(\cdot)})$ and $e(\cdot)$ both take $\hy_{\sF{S}}$ as argument is irrelevant: the bound holds for any $\tilde{y}\in \Hc$.} While $\sF{S}\in(\Xc\times\Yc)^m$ is also drawn independently, by assumption, we cannot quite conclude the same of $e_{\sF{S}}(c_{\hy_{\sF{S}}})$ because (as discussed above) with respect to the $c_{\hy_{\sF{S}}}$-induced measure on $\R$, the  sequence $(c_{\hy_{\sF{S}}}(z_1),\ldots,c_{\hy_{\sF{S}}}(z_m))$ is not. We may, however, suppose that a \textit{consequence} of independence holds, namely that \begin{equation}\label{eq:approxEquality} |\Eb(c_{\hy_{\sF{S}}}) - e_{\sF{S}}(\hy_{\sF{S}})|< \e/2,\end{equation} and use this (possibly counterfactual) supposition to test its truth. While possibly counterintuitive, a bound of the form in \eqref{eq:approxEquality} is exactly what we desire from a generalizing model $\hy_{\sF{S}}$. 

We first collect some definitions.
\begin{definition}\label{def:overfitting}
	Let $\sF{S}\sim_{iid}\Pb_{\Xc\times\Yc}$ and $\hy_{\sF{S}}\in \Hc$. We say that $\hy_{\sF{S}}$ $\e$-\textit{overfits} $\sF{S}$ if $$e_{\sF{S}}(\hy_{\sF{S}}) < \Eb(c_{\hy_{\sF{S}}}) - \e.$$

	\end{definition} This definition may be extended to $\e$-\textit{underfitting} with inequality $e_{\sF{S}}(\hy_{\sF{S}})>\Eb(c_{\hy_{\sF{S}}}) +\e$ and $\e$-generalization if model $\hy_{\sF{S}}$ neither $\e$-overfits nor $\e$-underfits. 
	
	Irrespective of threshold value $\e$, we call $\e_\Hc(\sF{S})\defeq e_{\sF{S}}(\hy_{\sF{S}}) - \Eb(c_{\hy_{\sF{S}}})$ the \textit{overfitting margin} of $\sF{S}$ and $\mu_\Hc\defeq \disg_{\Zc^m}\e_\Hc(\sF{S})d\Pb_{\Zc^m}$ the \textit{mean overfitting margin}. 
\begin{definition}\label{def:overfittingMargin}
	Irrespective of threshold value $\e$, we call $\e_\Hc(\sF{S})\defeq |e_{\sF{S}}(\hy_{\sF{S}}) - \Eb(c_{\hy_{\sF{S}}})|$ the \textit{overfitting margin} of $\sF{S}$ and $\mu_\Hc\defeq \disg_{\Zc^m}\e_\Hc(\sF{S})d\Pb_{\Zc^m}(\sF{S})$ the \textit{mean overfitting margin}.  We also define \textit{empirical [mean] overfitting margin} by $\e_\Hc(\sF{S},\sF{S}')\defeq \big|e_{\sF{S}}(\hy_{\sF{S}}) - e_{\sF{S}'}(\hy_{\sF{S}})\big|$. 
\end{definition}
Notice that the absolute value in \cref{def:overfittingMargin} subsumes both overfitting and underfitting as defined in \cref{def:overfitting}. The  phenomenon of underfitting as typically understood is not quite represented by these definitions: a class is ordinarily said to underfit if something the epsilon in $\Eb(c_{\hy_{\sF{S}}}) = \inf_{\tilde{y}\in \Hc}\Eb(c_{\hy_{\tilde{y}}}) + \e$ is large. 

\begin{prop}[Test for Overfitting]\label{prop:theProp}
	Suppose that  model $\hy_{\sF{S}}$ $\e/2$-generalizes (\cref{eq:approxEquality}). Then \begin{equation}\label{eq:testInequality}
 \Pb_{\Zc^{m'}}\left(\big|e_{\sF{S}}(\hy_{\sF{S}}) - e_{\sF{S}'}(\hy_{\sF{S}})\big|> \e\right)\leq 2e^{-\frac{\e^2m'}{2}}.	
 \end{equation}

\end{prop}  
Therefore, the null hypothesis that trained model $\hy_{\sF{S}}$ $\frac{\e}{2}$-generalizes may be tested using probability bound \cref{eq:testInequality}. 
 See \cite[\S11]{uml} for a similar result.

Notice that use of holdout data for evaluation by itself provides an absolute approximation of performance, while in tandem with training data, we gain quantified (un)certainty specifically about generalization. Finally, the probability in \eqref{eq:testInequality} depends on the size of validation data, but not on the size of training data. This conclusion is correct: while we would like more training data to correlate with higher likelihood of  performance, the problem in \cref{subsection:overfitting} indicates that such intuition may not find a straightforward grounding in   probability. Presumably, one may be less inclined to hypothesize satisfactory model performance when training with little data. The intuition finds security in PAC learnability, absent which there is no obvious guaranteed connection between size of (training) data and performance; we discuss this issue further in \cref{subsection:pLearnability}.

\subsection{Estimating for Overfitting}

 \begin{figure}[h!]
\centerline{\includegraphics[scale=.15]{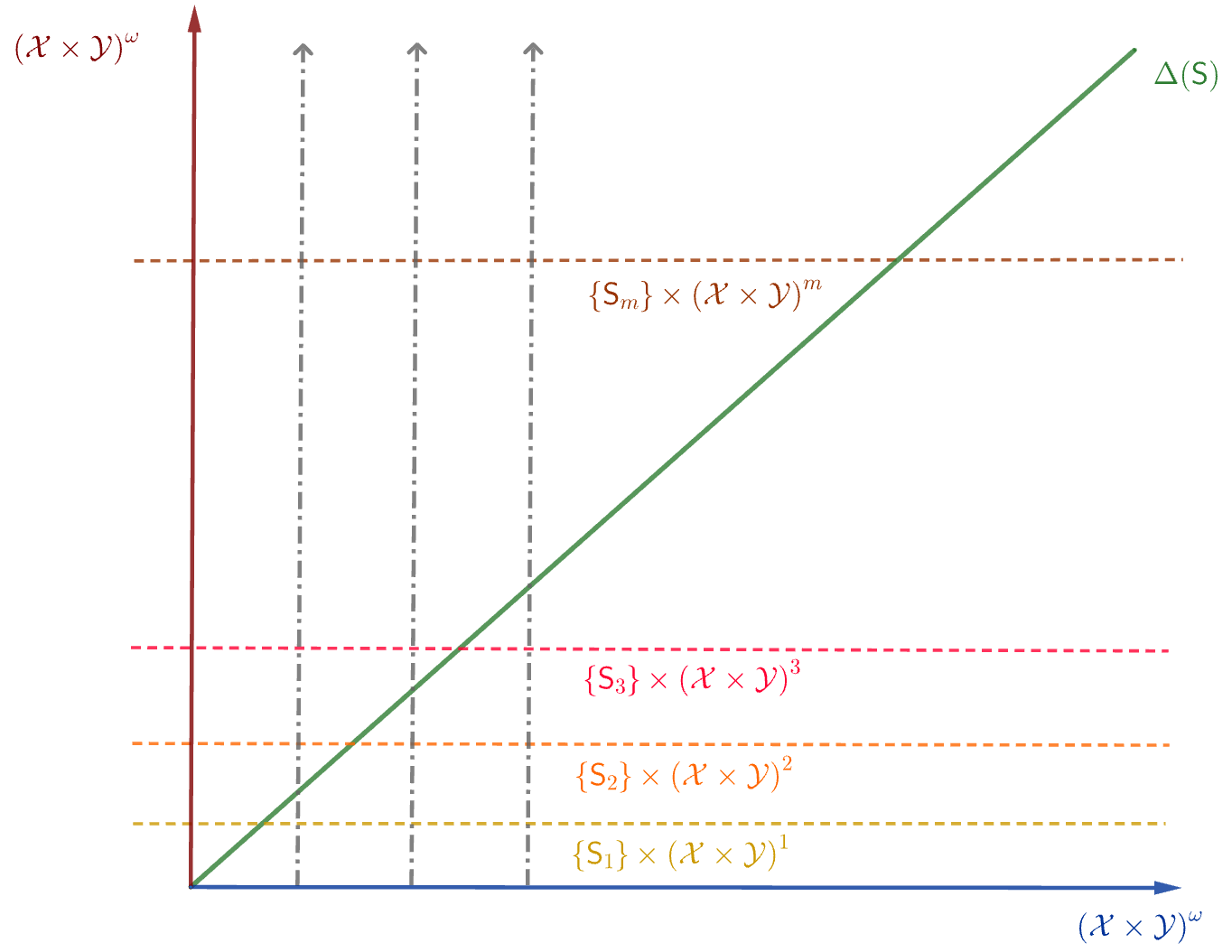}}\caption{Geometry for map $e_{(\cdot)}(\cdot):(\Xc\times\Yc)^{2\omega}\rightarrow\R$} \label{fig:figure3}
\end{figure}

With this review in hand, recall overfitting margin  $\e_{\Hc}(\sF{S}) = \big| \Eb(c_{\hy_{\sF{S}}}) - e_{\sF{S}}(\hy_{\sF{S}}) \big|$  and the empirical overfitting margin $\e_\Hc(\sF{S,S'}) =\big|e_{\sF{S}}(\hy_{\sF{S}}) - e_{\sF{S}'}(\hy_{\sF{S}})\big|$(\cref{def:overfittingMargin}). 
Typical concentration bounds guarantee with high probability that $\left|\mu_\Hc - \frac{1}{k}\diss_{j=1}^k \e_\Hc(\sF{S}_j)\right| < \g$ for $\g>0$, thus leading to a test which we can use for estimating the overfitting margin $\m_\Hc$ (\cref{def:overfitting}).
For fixed $\sF{S}$, we expect $\e_\Hc(\sF{S,S'})$ to be close to $\e_\Hc(\sF{S})$ and therefore that the empirical mean of empirical overfitting margins $\frac{1}{k}\sum_{j=1}^k\e_\Hc(\sF{S}_j,\sF{S}')$ be near the overfitting margin $\m_\Hc$.  The next result quantifies how close. 

\begin{prop} Let $\sF{S}_1,\ldots,\sF{S}_k\sim_{iid}\Pb_{\Zc^m}$ and $\sF{S}'\sim \Pb_{\Zc^{m'}}$, and suppose that $m'>k+\frac{2\log(k/\d)}{\e^2}$.  Then 
\begin{equation}\label{eq:overfittingMarginEst}
\Pb_{\Zc^{km+m'}}\left(\left|\frac{1}{k}\diss_{j=1}^k\e_\Hc(\sF{S}_j,\sF{S}')-\mu_\Hc\right|>\e\right) \leq 4e^{-k\e^2/2}. 
\end{equation} 
\end{prop}

\begin{proof}
We add zero and use the triangle inequality: 
\begin{equation}
	\begin{array}{ll}\left|
	\frac{1}{k}\diss_{j=1}^k\e_\Hc(\sF{S}_j,\sF{S}')-\mu_\Hc\right| & 
	 = \left|\frac{1}{k}\diss_{j=1}^k\left|e_{\sF{S'}}(\hy_{\sF{S}_j}) -  \Eb(c_{\hy_{\sF{S}_j}})+\Eb(c_{\hy_{\sF{S}_j}})- e_{\sF{S}_j}(\hy_{\sF{S}_j})\right|-\mu_\Hc\right|\\
	& \leq   \left|\frac{1}{k}\diss_{j=1}^k\left(\left|e_{\sF{S'}}(\hy_{\sF{S}_j}) -  \Eb(c_{\hy_{\sF{S}_j}})\right|+\left|\Eb(c_{\hy_{\sF{S}_j}})- e_{\sF{S}_j}(\hy_{\sF{S}_j})\right|\right)-\mu_\Hc\right|\\ 
	& \leq      \frac{1}{k}\diss_{j=1}^k\left|e_{\sF{S}'}(\hy_{\sF{S}_j}) -  \Eb(c_{\hy_{\sF{S}_j}})\right|+\left|\diss_{j=1}^k\left|\Eb(c_{\hy_{\sF{S}_j}})- e_{\sF{S}_j}(\hy_{\sF{S}_j})\right|-\mu_\Hc\right|\\ 
	& =  \frac{1}{k}\diss_{j=1}^k\left|e_{\sF{S}'}(\hy_{\sF{S}_j}) -  \Eb(c_{\hy_{\sF{S}_j}})\right|+\left|\diss_{j=1}^k\e_\Hc(\sF{S}_j)-\mu_\Hc\right|\\ 
	\end{array}
\end{equation}	
Therefore, $$\begin{array}{ll}\left\{\left|\frac{1}{k}\diss_{j=1}^k\e_\Hc(\sF{S}_j,\sF{S}')-\mu_\Hc\right|>\e\right\} & \subseteq  \left\{\frac{1}{k}\diss_{j=1}^k\left|e_{\sF{S}'}(\hy_{\sF{S}_j}) -  \Eb(c_{\hy_{\sF{S}_j}})\right|>\e/2\right\}\cup \left\{\left|\diss_{j=1}^k\e_\Hc(\sF{S}_j)-\mu_\Hc\right|>\e/2\right\}
\\ 
& \subseteq \left\{\dissup_{j\leq k} \left|e_{\sF{S}'}(\hy_{\sF{S}_j}) -  \Eb(c_{\hy_{\sF{S}_j}})\right|>\e/2\right\}\cup \left\{\left|\diss_{j=1}^k\e_\Hc(\sF{S}_j)-\mu_\Hc\right|>\e/2\right\},
\end{array}
$$ so the left hand side of \eqref{eq:overfittingMarginEst} is bounded by $2ke^{-m'\e^2/2}+2e^{-k\e^2/2}$, using the union bound twice, and Hoeffding inequality for both terms on the right. When $m'>k+\frac{2\log(k/\d)}{\e^2}$, then $$2e^{-k\e^2/2}\left(ke^{-(m'-k)\e^2/2}+1\right) \leq 4 e^{-k\e^2/2}.$$ \end{proof}

\subsection{Interpreting the Output}\label{subsection:interpreting}
Overfitting is a heuristic notion which suggests a model has fit the data and not the distribution which generated it. On closer inspection, however, the test we propose does not provide indication of \textit{only} overfitting. In fact, the supposition of generalization is  one with respect to a certain (fixed) distribution; this test thus additionally assumes that the test data $\sF{S}'\sim_{iid}\Pb_{\Xc\times\Yc}$ as well. It may not. For there may be some form of distributional shift according to which $\sF{S}'\sim_{iid} \Pb'_{\Xc\times\Yc}$, in which case we cannot guarantee the bound in  \eqref{eq:testInequality}, at least not if the expectation $\Eb(c_{\hy_{\sF{S}}})$ is computed with respect to the original measure $d\Pb_{\Xc\times\Yc}$. In other words, instantiation  of   event $\left\{\big|e_{\sF{S}}(\hy_{\sF{S}})- e_{(\cdot)}(\hy_{\sF{S}})\big| > \e\right\}$  by inequality $|e_{\sF{S}}(\hy_{\sF{S}})- e_{\sF{S}'}(\hy_{\sF{S}})| > \e$  may suggest: 
\begin{enumerate}
	\item an unlikely sample  $\sF{S}'$  was received (all the hypotheses hold),
	\item $\hy_{\sF{S}}$ does not generalize $\Pb_{\Xc\times\Yc}$ with respect to $c_{(\cdot)}$ (overfitting), or
	\item $\sF{S}'\not\sim_{iid}\Pb_{\Xc\times\Yc}$ (possible distributional shift). 
\end{enumerate}
It is important when running a statistical test to respect the scope of what it purports to evaluate: namely, \textit{if} a set of assumptions hold---in this case  1.\ that $\hy_{\sF{S}}$ $\frac{\e}{2}$-generalizes (\cref{eq:approxEquality}) and 2.\ $\sF{S}'\sim_{iid} \Pb_{\Xc\times\Yc}$---then the probability that a certain kind of event occurs is bounded by some value which is explicitly calculable. Realization of the unlikely and unlucky event by $\sF{S}'$ can either mean $\sF{S}'$ really is unlucky or that one of the assumptions fails.

While this test is expressed with respect to the cost function $c_{\tilde{y}}$ or $c_{\hy_{\sF{S}}}$,  it need not be so limited. In fact, any map $f:\Xc\times\Yc\rightarrow\R$ may be used to probe the distribution, substituting the appropriate concentration inequality depending on the range of $f$. When $f(\Xc\times\Yc)$ is bounded, we may rely on a version of Hoeffding, which converges exponentially. Subsequent work will investigate the use of \textit{random projections} to examine distribution shift and uncertainty quantification, as a means of testing to eliminate or isolate the above obfuscating condition \#3.

\subsection{Loosening Uniform Bounds}\label{subsection:pLearnability}

\begin{figure}[h!]
\centerline{\includegraphics[scale=.75]{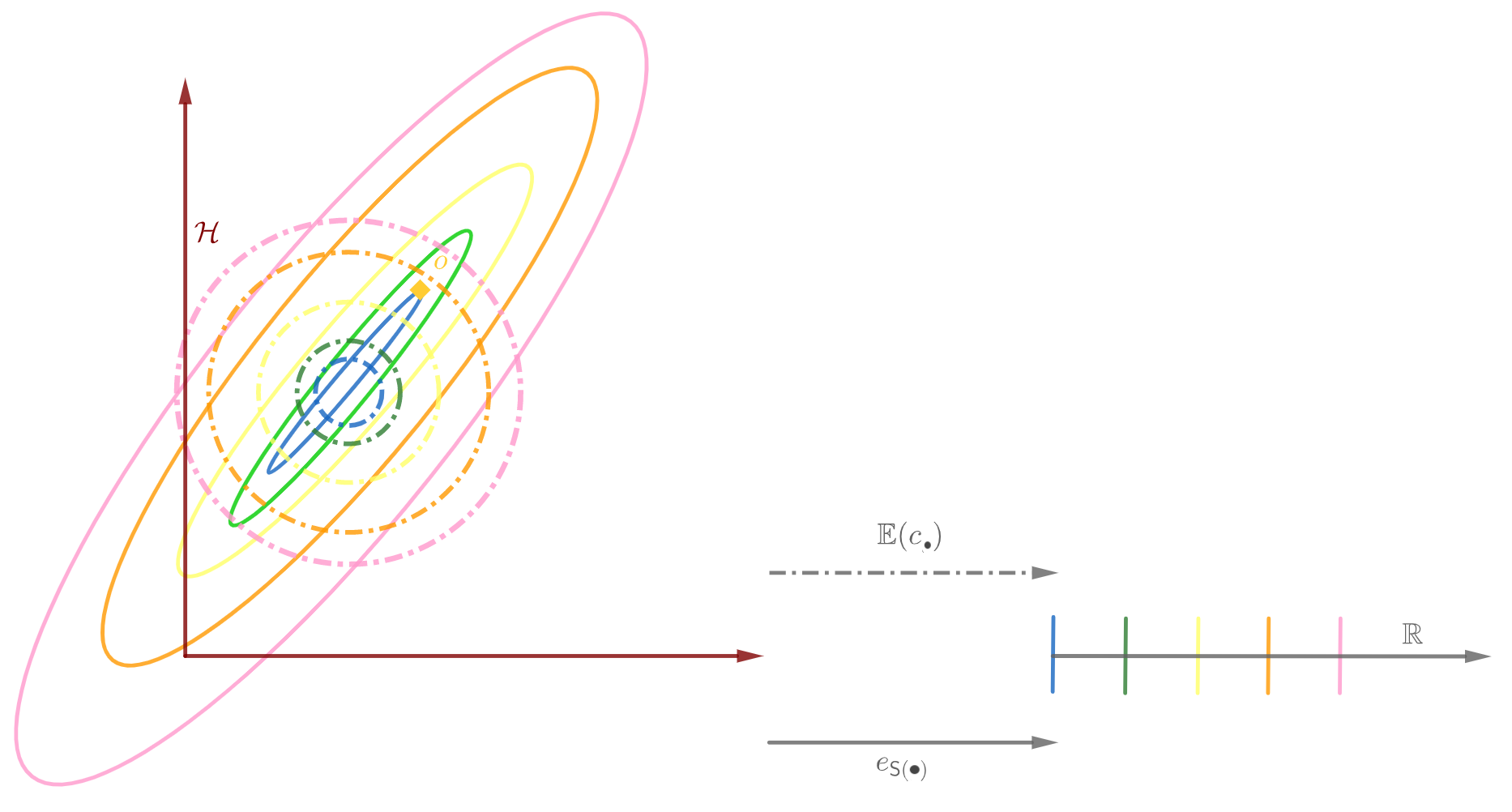}}\caption{Model $o$ overfits.}\label{fig:profiles}
\end{figure}
We conclude with commentary on the merits of this test. The bound in \cref{eq:testInequality} is perhaps unsurprising and at first glance offers little upside beyond the performance guarantee as provided by \cref{eq:expectedApprox}, which guarantees approximation of empirical mean using holdout data to the true mean. Indeed, on may stipulate that overfitting, in the sense of \cref{def:overfitting}, induces little cause for concern: as long as performance $\Eb(c_{\hy_{\sF{S}}})$ is ``good enough,'' (as approximated by $\e_{\sF{S}'}(\hy_{\sF{S}})$) it may not particularly matter whether or that training data performance matches a model's generalization performance. On the other hand, guarantees of the sort which PAC learnability provides ensure that the output of a training algorithm is near optimal in a hypothesis class. In the presence of overfitting, one may not know whether better than `good enough' is achievable. Generalization \textit{with training data} provides confidence that empirical risk minimization \eqref{eq:erm} approximately realizes risk minimization \eqref{eq:generalization} \textit{in the absence of uniform (PAC) guarantees}. The test is a workable mechanism for checking that there is little gap between performance a hypothesis class may achieve on data and on the data's distribution (\cref{fig:profiles}). 

We underscore the point. PAC guarantees ensure not only that an algorithm will return an optimal (in the hypothesis class) model, but that the sample complexity with respect to which the algorithm is expected to reliably work is \textit{independent of distribution}. Guarantees of this form are helpful in providing confidence ahead of time that the learning endeavor is not misguided. On the other hand, practitioners often engage in the tackling the learning problem irrespective of knowledge or other assurances that their class is PAC learnable. Moreover, PAC learnability does not cover the intermediate case that some distributions may require a larger sample complexity (some tasks are harder to learn than others), and that there is no uniform bound over all measures. Still, assurance that the output of training generalizes does not \textit{require} that the hypothesis class is PAC learnable, i.e.\ that uniform bounds hold. Rather: uniform bounds, when they exist, provide a conceptual framework and analytic setting wherein a class of results may be generated, in the absence of which, we would nevertheless like to  be able to say \textit{something}. 

Consider, for example, \cref{fig:profiles} which compares level sets for $\Eb(c_{(\cdot)})$ and $e_{\sF{S}}(\cdot)$. Learnability, as described by uniform convergence and notions  of representability (c.f.\ \cite[\S4.1]{uml}), guarantees that these profiles roughly track each other,\footnote{Again: and PAC learnability guarantees that the tracking is independent of distribution.} which is \textit{sufficient} for generalization of output model $\hy_{\sF{S}}$: if the value of $\Eb(c_{(\cdot)})$ and $e_{\sF{S}}(\cdot)$ are roughly approximate \textit{everywhere} in $\Hc$, then they certainly are at a particular point. On the other hand, learnability objectives ultimately desire generalization of the output, namely that $\Eb(c_{(\cdot)})$ and $e_{\sF{S}}(\cdot)$ are roughly approximate \textit{at} $\hy_{\sF{S}}$; how they compare in other regions of $\Hc$ is immaterial.

\appendix

\section{Hoeffding's Inequality for Statistical Hypothesis Testing}\label{appendix:Hoeffding}
Hoeffding's inequality gives a probability bound for independent sample $\sF{S}=(x_1,\ldots,x_m)\sim_{iid}\Pb_{\Xc}$ when $\Xc = [0,1]$, namely: 

\begin{equation}\label{eq:Hoeffding} \Pb_\Xc\left(\left|\disg_\Xc xd\Pb_\Xc(x)- \frac{1}{|\sF{S}|}\diss_{x\in \sF{S}}\right|\right)< 2e^{-2\e^2|\sF{S}|}.\end{equation}
Therefore, given any two of confidence specification $\d\in (0,1)$,  data set sized $|\sF{S}| = m$, and precision bound $\e\in (0,1)$, one may readily solve for the third. 

Proof of its verity and other applications may be found in various probability texts (\cite{ce}, \cite{uml}, \cite{fml}).

\bibliographystyle{ieeetr}
\bibliography{phdref}


\end{document}